\documentclass[]{article}
\usepackage{proceed2e}

\usepackage[textwidth=7.25in,textheight=10in]{geometry}

\usepackage{amsthm,amsmath, amssymb,bm, cases, mathtools, thmtools}

\usepackage{times}

\title
{Training generative neural networks via Maximum Mean Discrepancy optimization}

\author{
{\bf Gintare Karolina Dziugaite}  \\
University of Cambridge \\
\And
{\bf Daniel M. Roy}  \\
University of Toronto   \\
\And
{\bf Zoubin Ghahramani}   \\
University of Cambridge \\
}

\usepackage{graphicx}\graphicspath{{figures/}}
\usepackage{epstopdf}
\usepackage{algpseudocode}
\usepackage{algorithm}
\usepackage[style=numeric,
       minnames=1,
       maxnames=99,
       maxcitenames=3,
       natbib=true,
       sortcites=true,
       firstinits=true,
       backend=bibtex]{biblatex}       
\bibliography{biblio}

\usepackage[bf,small]{caption} %
\usepackage{subcaption}

\usepackage[linktoc=all,linkbordercolor={1 1 0}]{hyperref}
\usepackage[capitalize]{cleveref}
\crefname{lemma}{Lemma}{Lemmas}
\crefname{algorithm}{Algorithm}{Algorithms}
\crefname{theorem}{Theorem}{Theorems}

\usepackage[usenames]{color}
\newcommand{\LATER}[1]{\error}
\newcommand{\fLATER}[1]{\error}
\newcommand{\TBD}[1]{\error}
\newcommand{\fTBD}[1]{}
\newcommand{\PROBLEM}[1]{\error}
\newcommand{\fPROBLEM}[1]{}

\newtheorem{theorem}{Theorem}
\newtheorem{lemma}{Lemma}
\newtheorem{corollary}{Corollary}

\theoremstyle{definition}
\newtheorem{definition}{Definition}

\theoremstyle{remark}

\def\[#1\]{\begin{align}#1\end{align}}
\def\*[#1\*]{\begin{align*}#1\end{align*}}
\newcommand{\dee}{\mathrm{d}}
\newcommand{\truedist}{p_{\mathrm{data}}} 
\newcommand{\noisedist}{p_{\mathrm{noise}}} 
\newcommand{\TRV}{X} %
\newcommand{\NRV}{W} %
\newcommand{\RSS}{\mathcal W}

\newcommand{\GRV}{Y}
\newcommand{\grv}{y}
\newcommand{\trv}{x}
\newcommand{\nrv}{w}
\newcommand{\RKHS}{\mathcal H}
\newcommand{\OSpace}{\mathbb{X}}
\newcommand{\NSpace}{\mathbb{W}}

\newcommand{\genf}{G_\theta}
\newcommand{\empopt}{\hat\theta}
\newcommand{\classopt}{\theta^*}
\newcommand{\pbound}{\delta_\varepsilon}

\newcommand{\fatdim}{\text{fat}_{\varepsilon}}
\newcommand{\tvec}{\underline{X}}
\newcommand{\nvec}{\underline{W}}
\newcommand{\copynoise}{\zeta}

\newcommand{\radRV}{\sigma}

\newcommand{\Reals}{\mathbb{R}}
\newcommand{\Nats}{\mathbb{N}}
\newcommand{\MMD}{\mathrm{MMD}}
\newcommand{\MMDu}{\MMD_u}
\newcommand{\GKZclass}{\mathcal{G}_{k + }}
\newcommand{\GKclass}{\mathcal{G}_k}
\newcommand{\GKclassX}{\GKclass^{\OSpace}} 
\newcommand{\GKZclassX}{\GKZclass^{\OSpace}}

\newcommand{\EE}{E}%

\newcommand{\dist}{\sim}

\begin{document}
\maketitle

\begin{abstract}
We consider training a deep neural network to generate samples from an unknown distribution given i.i.d.\ data. We frame learning as an optimization minimizing a two-sample test statistic---informally speaking, a good generator network produces samples that cause a two-sample test to fail to reject the null hypothesis. As our two-sample test statistic, we use an unbiased estimate of the maximum mean discrepancy, which is the centerpiece of the nonparametric kernel two-sample test proposed by Gretton et al.~\citep{Gretton:2012}. We compare to the \emph{adversarial nets} framework introduced by Goodfellow et al.~\citep{Goodfellow:2014}, in which learning is a two-player game between a generator network and an adversarial discriminator network, both trained to outwit the other. From this perspective, the MMD statistic plays the role of the discriminator. In addition to empirical comparisons, we prove bounds on the generalization error incurred by optimizing the empirical MMD.
\end{abstract}

\newcommand{\Data}{\mathcal P}
\newcommand{\Noise}{\mathcal N}
\section{Introduction}

In this paper, we consider the problem of learning generative models from i.i.d.\ data with unknown distribution $\Data$. 
We formulate the learning problem as one of finding a function $G$, called the \emph{generator}, 
such that, given an input $Z$ drawn from some fixed \emph{noise} distribution $\Noise$, 
the distribution of the output $G(Z)$ is close to the data's distribution $\Data$.  Note that, given $G$ and $\Noise$, we can easily generate new samples despite not having an explicit representation for the underlying density.

We are particularly interested in the case where the generator is a deep neural network whose parameters we must learn.  Rather than being used to classify or predict, these networks transport input randomness to output randomness, thus inducing a distribution.
The first direct instantiation of this idea is due to \citet{MacKay94}, 
although MacKay draws connections even further back to the work of \citet{Saund89} and others on autoencoders, suggesting that generators can be understood as decoders.
MacKay's proposal, called \emph{density networks}, uses multi-layer perceptrons (MLP) as generators and learns the parameters by approximating Bayesian inference.

Since MacKay's proposal, there has been a great deal of progress on learning generative models, especially over high-dimensional spaces like images.  Some of the most successful approaches have been based on restricted Boltzmann machines %
\citep{SalakhutdinovH09}
and deep Boltzmann networks %
\citep{HintonSalakhutdinov2006b}.  
A recent example is the Neural Autoregressive Density Estimator due to \citet{UriaML13a}.  
An indepth survey, however, is beyond the scope of this article.

This work builds on a proposal due to \citet{Goodfellow:2014}.
Their \emph{adversarial nets} framework 
takes an indirect approach to learning deep generative neural networks: a discriminator network is trained to recognize the difference between training data and generated samples, while the generator is trained to confuse the discriminator.  The resulting two-player game is cast as a minimax optimization of a differentiable objective and solved greedily by iteratively performing gradient descent steps to improve the generator and then the discriminator.

Given the greedy nature of the algorithm, \citet{Goodfellow:2014} give a careful prescription for balancing the training of the generator and the discriminator.  In particular, two gradient steps on the discriminator's parameters are taken for every iteration of the generator's parameters.  It is not clear at this point how sensitive this balance is as the data set and network vary.
In this paper, we describe an approximation to adversarial learning that replaces the adversary with
a closed-form nonparametric two-sample test statistic based on the Maximum Mean Discrepancy (MMD),
which we adopted from the kernel two sample test \citep{Gretton:2012}. 
We call our proposal \emph{MMD nets}.\footnote{In independent work reported in a recent preprint, Li, Swersky, and Zemel \citep{LSZ} also propose to use MMD as a training objective for generative neural networks.  We leave a comparison to future work.} 
We give bounds on the estimation error incurred by optimizing an empirical estimator rather than the true population MMD and give some illustrations on synthetic and real data.

\section{Learning to sample as optimization}

\begin{figure*}[t]
\centering
\includegraphics[width=\linewidth]{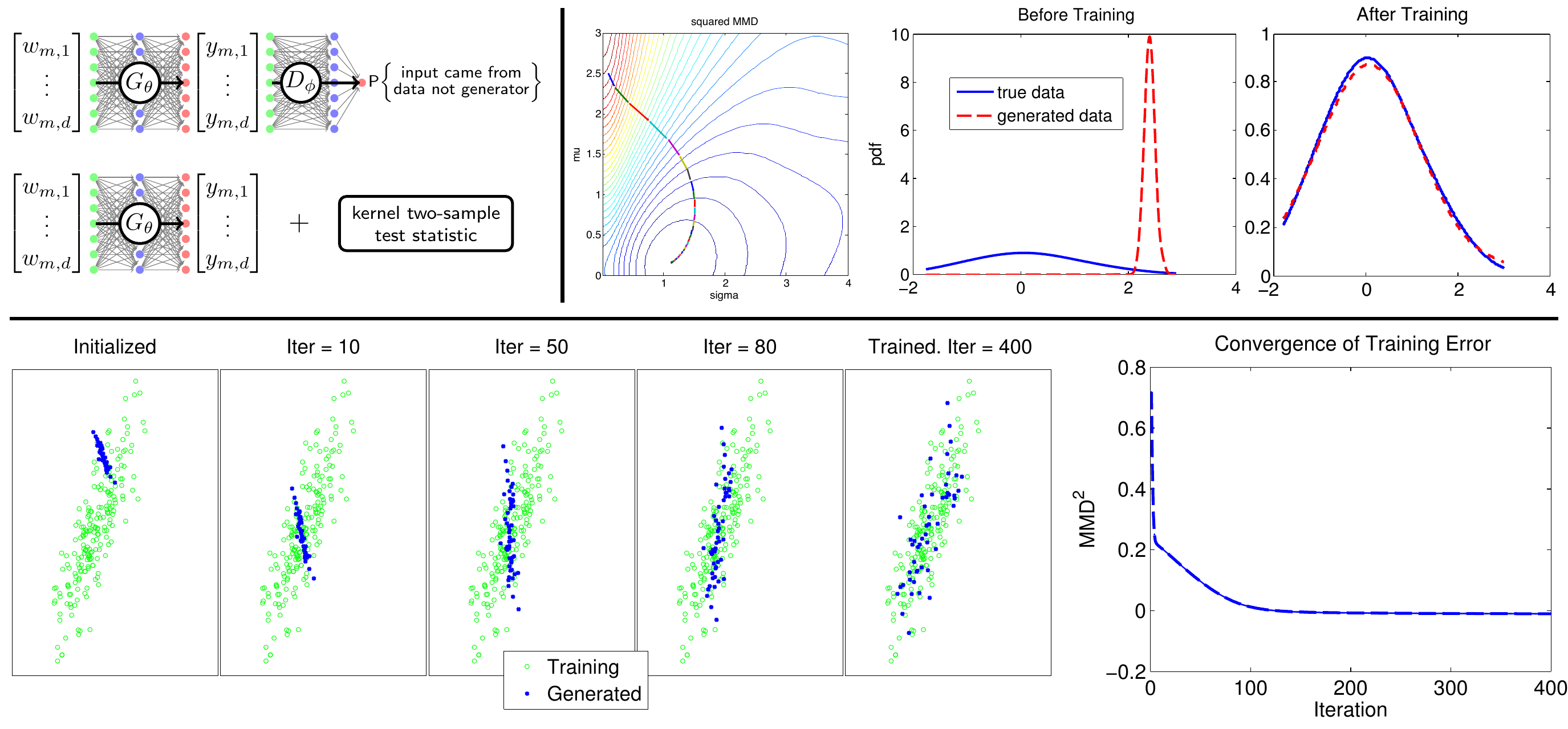}
\caption{
{\bf (top left)} Comparison of adversarial nets and MMD nets.
{\bf (top right)}
Here we present a simple one-dimensional illustration of optimizing a generator via MMD. Both the training data and noise data are Gaussian distributed and we consider the class of generators given by $G_{(\mu,\sigma)}(w) = \mu + \sigma w$.
The plot on the left shows the isocontours of the MMD-based cost function and the path taken by gradient descent.  On right, we show the distribution of the generator before and after a number of training iterations, as compared with the data generating distribution.  Here we did not resample the generated points and so we do not expect to be able to drive the MMD to zero and match the distribution exactly.
{\bf (bottom)}
The same procedure is repeated here for a two-dimensional dataset.  On the left, we see the gradual alignment of the Gaussian-distributed input data to the Gaussian-distributed output data as the parameters of the generator $G_\theta$ are optimized.  The learning curve on the right shows the decrease in MMD obtained via gradient descent. 
}
\label{fig:advVSmmd}
\end{figure*}

It is well known that, for any distribution $\Data$ 
and any continuous distribution $\Noise$ on sufficiently regular spaces
$\OSpace$ and $\NSpace$, respectively,
there is a function $G : \NSpace \to \OSpace$, such that $G(W) \dist \Data$ when $W \dist \Noise$. (See, e.g., \citep[][Lem.~3.22]{FMP2}.)
In other words, we can transform an input from a fixed input distribution $\Noise$ through a deterministic function, producing an output whose distribution is $\Data$.
For a given family $\{G_\theta\}$ of functions $\NSpace \to \OSpace$, called \emph{generators},
we can cast the problem of learning a generative model as an optimization
\[ \label{eq:optimizedivergence}
\arg\min_\theta \delta(\Data,G_\theta(\Noise)),
\]
where $\delta$ is some measure of discrepancy and $G_\theta(\Noise)$ is the distribution of $G_\theta(W)$ when $W \dist \Noise$. In practice, we only have i.i.d.\ samples $X_1,X_2,\dotsc$ from $\Data$, and so we optimize an empirical estimate of $\delta(\Data,G_\theta(\Noise))$.

\newcommand{\AN}{\mathrm{AN}}
\subsection{Adversarial nets}

Adversarial nets \citep{Goodfellow:2014} can be cast within this framework:
Let $\{D_\phi\}$ be a family of functions $\OSpace \to [0,1]$, called \emph{discriminators}.
We recover the adversarial nets objective with the discrepancy
\*[
\delta_{\AN}(\Data,G_\theta(\Noise)) = \max_{\phi} \EE \bigl [\log D_\phi(X) + \log (1 - D_\phi(Y)) \bigr],
\*]
where
$X \dist \Data$ and $Y \dist G_\theta(\Noise)$.
In this case, \cref{eq:optimizedivergence} becomes
\*[
	\min_{\theta} \max_{\phi}\, V(\genf,D_\phi) 
\*]
where 
\*[
V(\genf,D_\phi) = \EE \bigl [ \log D_\phi(\TRV)  
                            +  \log (1-D_\phi(\genf(\NRV))) \bigr ]
\*]
for $X \dist \Data$ and $W \dist \Noise$. 
The output of the discriminator $D_\phi$ can be interpreted as the probability it assigns to its input being drawn from $\Data$, and so $V(\genf,D_\phi)$ is the expected log loss incurred when classifying the origin of a point equally likely to have been drawn from $\Data$ or $\genf(\Noise)$.  
Therefore, optimizing $\phi$ maximizes the probability of distinguishing samples from $\Data$ and $\genf(\Noise)$.  
Assuming that the optimal discriminator exists for every $\theta$, the optimal generator $G$ is that whose output distribution is closest to $\Data$, as measured by the Jensen--Shannon divergence, which is minimized when $\genf(\Noise) = \Data$.

In \citep{Goodfellow:2014}, the generators $\genf$ and discriminators $D_\phi$ are chosen to be multilayer perceptrons (MLP). 
In order to find a minimax solution, they propose taking alternating gradient steps along  $D_\phi$ and $\genf$.
Note that the composition $D_\phi(\genf(\cdot))$ that appears in the value function is yet another (larger) MLP. 
This fact permits the use of the back-propagation algorithm to take gradient steps.

\subsection{MMD as an adversary}

In their paper introducing adversarial nets, \citet{Goodfellow:2014} remark that a balance must be struck between optimizing the generator and optimizing the discriminator.
In particular, the authors suggest
$k$ maximization steps for every one minimization step to ensure that $D_\phi$ is well synchronized with $G_\theta$ during training.
A large value for $k$, however, can lead to overfitting.
In their experiments, for every step taken along the gradient with respect to $\genf$, they take two gradient steps with respect to  $D_\phi$ to bring $D_\phi$ closer to the desired optimum (Goodfellow, pers.\ comm.).

It is unclear how sensitive this balance is.  Regardless, while adversarial networks deliver impressive sampling performance, the optimization takes approximately 7.5 hours
to train on the MNIST dataset 
running on a nVidia GeForce GTX TITAN GPU with 6GB RAM.
Can we potentially speed up the process with a more tractable choice of adversary?

Our proposal is to replace the adversary with the kernel two-sample test introduced by \citet{Gretton:2012}.
In particular, we replace the family of discriminators with a family $\RKHS$ of test functions $\OSpace \to \Reals$, closed under negation, and use the maximum mean discrepancy 
between $\Data$ and $G_\theta(\Noise)$ over $\RKHS$, given by
\[\label{MMDopt}
\delta_{\MMD_\RKHS} (\Data,G_\theta(\Noise)) = \sup_{f \in \RKHS} \EE [ f(X)] - \EE[f(Y)],
\]
where $X \dist \Data$ and $Y \dist G_\theta(\Noise)$.
See \cref{fig:advVSmmd} for a comparison of the architectures of adversarial and MMD nets.

While \cref{MMDopt} involves a maximization over a family of functions, 
\citet{Gretton:2012} show that it can be solved in closed form when $\RKHS$ is a reproducing kernel Hilbert space (RKHS).  

More carefully, 
let $\RKHS$ be a reproducing kernel Hilbert space (RKHS) of real-valued functions on $\Omega$ and let $\langle\cdot, \cdot \rangle_{\RKHS}$ denote its inner product. By the reproducing property it follows that there exists a \textit{reproducing kernel} $k \in \RKHS$ such that every $f \in \RKHS$ can be expressed as
\begin{equation}\label{RKHSequation}
	f(x) = \langle f, k(\cdot, x) \rangle_{\RKHS} = \sum \alpha_i k(x,x_i)
\end{equation}  
The functions \emph{induced by a kernel $k$} are those functions in the closure of the span of the set 
$ \{ k(\cdot,x): x \in \Omega \}$, which is necessarily an RKHS.
Note, that for every positive definite kernel there is a unique RKHS $\RKHS$ such that every function in $\RKHS$ satisfies \cref{RKHSequation}.

Assume that $\OSpace$ is a nonempty compact metric space and $\mathcal{F}$ a class of functions $f: \OSpace \rightarrow \Reals$. Let $p$ and $q$ be Borel probability measures on $\OSpace$, and let $X$ and $Y$ be random variables with distribution $p$ and $q$, respectively.
The \emph{maximum mean discrepancy} (MMD) between $p$ and $q$ is 
\[
\MMD(\mathcal{F},p,q) = \sup_{f \in \mathcal{F}} \EE [f(X)] - \EE [f(Y)]
\]
If $\mathcal{F}$ is chosen to be an RKHS $\RKHS$, 
then 
\[
\MMD^2(\mathcal F, p,q) = \| \mu_p - \mu_q \|_{\RKHS}^2
\]
where $\mu_p \in \RKHS$ is the \emph{mean embedding} of $p$,
given by
\[
 \mu_p = \int_{\OSpace} k(x,\cdot) \, p(\dee x) \in \RKHS 
\]
and satisfying, for all $f \in \RKHS$,
\*[
\EE[ f(X)] = \langle f, \mu_p\rangle_{\RKHS}.
\*]
The properties of $\MMD(\RKHS,\cdot,\cdot)$ depend on the underlying RKHS $\RKHS$.  For our purposes, 
it suffices to say that if we take $\OSpace$ to be $\Reals^D$ and consider the RKHS $\RKHS$ induced by Gaussian or Laplace kernels, then MMD is a metric,
and so the minimum of our learning objective is achieved uniquely by $\Data$, as desired.  
(For more details, see \citet{Fukumizu2008}.)

In practice, we often do not have access to $p$ or $q$.  Instead, we are given
independent i.i.d.\ data $X,X',X_1,\dotsc,X_N$ and $Y,Y',Y_1,\dotsc,Y_M$ fom $p$ and $q$, respectively, and would like to estimate the MMD.
\citet{Gretton:2012} showed that
	\begin{equation}\label{eq:MMD}
	\MMD^2[\RKHS,p,q] = \EE [k(X,X')- 2 k(X,Y) +  k(Y,Y')]
	\end{equation}
and then proposed an unbiased estimator
\begin{equation}  \label{eq:MMDempirical}
\begin{split}
\MMDu^2[\RKHS,X,Y] = 
\,& \frac{1}{N (N-1)} \sum_{n \neq n'} k(x_n,x_{n'}) \\
& + \frac{1}{M (M-1)} \sum_{m \neq m'}  k(y_m,y_{m'}) \\
& - \frac{2}{M N} \sum_{m=1}^{M} \sum_{n =1}^{N} k(x_n,y_m).
\end{split}
\end{equation}

\begin{algorithm}[t]
\caption{Stochastic gradient descent for MMD nets.}\label{stochasticalgorithm}
\begin{algorithmic}
\State \text{Initialize $M$, $\theta$, $\alpha$, $k$}
 \State \text{Randomly divide training set $\TRV$  into $N_{\text{mini}}$ mini batches}
\For{$i\gets 1, \text{number-of-iterations}$}
      \State \text{Regenerate noise inputs $\{\nrv_i\}_{i=1,...,M}$ every $r$ iterations}
	\For{$n_{\text{mini}} \gets 1, N_{\text{mini}}$}
	   \For{$m \gets 1, M$}
      		\State $\grv_m \gets G_{\theta} (\nrv_m)$  
   	   \EndFor 
   	   \State \text{compute the n'th minibatch's gradient $\nabla C^{(n)}$}
   	   \State \text{update learning rate $\alpha$ (e.g., RMSPROP)}
   	   \State $\theta\gets \theta - \alpha \nabla C_n$
	\EndFor 
\EndFor
\end{algorithmic}
\end{algorithm}

\section{MMD Nets}

With an unbiased estimator of the MMD objective in hand, we can now define our proposal, \emph{MMD nets}:
Fix a neural network $G_\theta$, where $\theta$ represents the parameters of the network.  
Let $W=(w_1,\dotsc,w_M)$ denote noise inputs drawn from $\Noise$, 
let $Y_\theta=(y_1,\dotsc,y_m)$ with $y_j = G_\theta(w_j)$ denote the noise inputs transformed by the network $G_\theta$,
and let
$\TRV = (\trv_1,..,\trv_N)$ denote the training data in $\Reals^D$.
Given a positive definite kernel $k$ on $\Reals^D$, we minimize $C(Y_\theta,X)$ as a function of $\theta$, 
where 
\begin{equation} \label{costfunctionMMD}
\begin{split}
C (\GRV_{\theta},\TRV)
&= \frac{1}{M (M-1)}\sum_{m \neq m'} k(\grv_m,\grv_{m'}) 
\\&\qquad- \frac{2}{M N} \sum_{m=1}^{M} \sum_{n =1}^{N} k(\grv_m,\trv_n).
\end{split}
\end{equation}
Note that $C(Y_\theta,X)$ is comprised of only those parts of the unbiased estimator that depend on $\theta$.

In practice, the minimization is solved by gradient descent, possibly on subsets of the data.  More carefully,
the chain rule gives us 
\begin{align}
\nabla C (\GRV_{\theta},\TRV) 
	&= \frac{1}{N}\sum_{n=1}^N  \sum_{m=1}^M \frac{\partial C_n(\GRV_{\theta},\TRV_n)}{\partial \grv_m} \frac{\partial G_{\theta} (\nrv_m)}{\partial \theta} ,
\end{align}
where
\begin{equation}
\begin{split}
 C_n(\GRV_{\theta},\TRV_n) 
 &= \frac{1}{M (M-1)} \sum_{m\neq m'} k(\grv_m,\grv_{m'})
 \\&\qquad- \frac{2}{M} \sum_{m=1}^{M} k(\grv_m,\trv_n).
\end{split}
\end{equation}
Each derivative $ \frac{\partial C_n(\GRV_{\theta},\TRV_n)}{\partial \grv_m} $ is easily computed for standard kernels like the RBF kernel. Our gradient $\nabla C(\GRV_{\theta},\TRV_n) $ depends on the partial derivatives of the generator with respect to its parameters, which we can compute using back propagation.

\section{Generalization bounds for MMD}

 MMD nets operate by minimizing an empirical estimate of the MMD.  This estimate is subject to Monte Carlo error and so the network weights (parameters) $\hat \theta$ that are found to minimize the empirical MMD may do a poor job at minimizing the exact population MMD.  We show that, for sufficiently large data sets, this estimation error is bounded, despite the space of parameters $\theta$ being continuous and high dimensional.

Let $\Theta$ denote the space of possible parameters for the generator $G_\theta$, 
let $\Noise$ be the distribution on $\RSS$  for the noisy inputs,
and let 
$p_\theta = G_\theta(\Noise)$ be the distribution of $G_\theta(\NRV)$ when $\NRV \dist \Noise$ for $\theta \in \Theta$.
Let $\empopt$ be the value optimizing the unbiased empirical MMD estimate, i.e.,
\[ \label{eq:eMMDminimizer}
\MMDu^2(\RKHS,\TRV,\GRV_{\empopt})
= \inf_{\theta} \MMDu^2(\RKHS,\TRV,\GRV_{\theta}),
\]
and let $\classopt$ be the value optimizing the population MMD, i.e.,
\[
\MMD^2(\RKHS,\truedist,p_{\classopt})  =  \inf_{\theta} \MMD^2(\RKHS,\truedist,p_{\theta}).
\]
We are interested in bounding the difference
\[
 \MMD^2(\RKHS,\truedist,p_{\empopt}) - \MMD^2(\RKHS,\truedist,p_{\classopt}).
\]

To that end, for a measured space $\mathcal{X}$,
write $L_\infty (\mathcal{X})$ for the space of essentially bounded functions on $\mathcal{X}$ and write $B(L_\infty (\mathcal{X}))$ for the unit ball under the sup norm, i.e., 
\*[
B(L_\infty (\mathcal{X})) = \{ f \colon\mathcal{X} \rightarrow \mathbb{R} \,: (\forall x\ \in \mathcal{X}) f(x) \in [-1,1]  \}.
\*]
The bounds we obtain will depend on a notion of complexity captured by the fat-shattering dimension:
\begin{definition}[Fat-shattering {\citep{Mendelson03}}]
Let $\mathcal{X}_N = \{x_1,\dots,x_N\} \subset \mathcal{X}$ and   $\mathcal{F} \subset B(L_\infty (\mathcal{X}))$. For every $\varepsilon > 0$, $\mathcal{X}_N$ is said to be \emph{$\varepsilon$-shattered by $\mathcal{F}$} if there is some function $h:\ \mathcal{X} \rightarrow \mathbb{R}$, such that  for every $I \subset \{1,\dots , N\}$ there is some $f_I \in \mathcal{F}$ for which 
\[
f_I (x_n) &\geq h(x_n) + \varepsilon \, \text{ if } \, n \in I, \\
f_I(x_n) &\leq h(x_n) - \varepsilon\, \text{ if } \, n \notin I.
\]
For every $\varepsilon$, the \emph{fat-shattering dimension of $\mathcal{F}$}, written $\fatdim(\mathcal{F})$, is defined as
\*[
\fatdim(\mathcal{F}) = \sup \left\{ |\mathcal{X}_N | : \, \mathcal{X}_N \subset \ \mathcal{X},\, \mathcal{X}_N \text{ is }\varepsilon \text{-shattered by } \mathcal{F}  \right\}
\*]
\end{definition}
We then have the following bound on the estimation error:
\begin{theorem}[estimation error]\label{maintheorem}
Assume the kernel is bounded by one.
Define 
\[
\GKZclass = \{ g = k(\genf (\nrv),\genf (\cdot)) : \, \nrv \in \RSS,\, \theta \in \Theta \} 
\] 
and
\[
\GKZclassX = \{ g = k(\trv,\genf (\cdot)) : \, \trv \in \OSpace,\, \theta \in \Theta \}.
\]
Assume  
there exists $\gamma_1,\gamma_2 > 1$ and $p_1, p_2 \in \Nats$ such that, for all $\varepsilon>0$, 
	it holds that $\fatdim (\GKZclass) \leq \gamma_1 \varepsilon^{-p_1}$ and  $\fatdim (\GKZclass^X ) \leq \gamma_2 \varepsilon^{-p_2}$.
Then with probability at least $1-\delta$, 
\[
 \MMD^2(\RKHS,\truedist,p_{\empopt}) < \MMD^2(\RKHS,\truedist,p_{\classopt}) + \varepsilon,
\]
with 
\[
\varepsilon 
 &= r(p_1,\gamma_1,M) + r(p_2,\gamma_2,M-1) + 12 M^{-\frac 1 2} \sqrt { \log \frac{2}{\delta}},
\]
where the rate $r(p,\gamma,N)$ is
\[ 
r(p,\gamma,M) = C_{p} \sqrt{\gamma}
\begin{cases}  
        M^{- \frac 1 2}        &\text{if } p<2,\\ 
        M^{- \frac 1 2} \log^{\frac 3 2} (M)   &\text{if } p = 2, \\ 
	M^{-\frac{1}{p}}   &\text{if } p>2,
\end{cases}
\] 
for constants $C_{p_1}$ and $C_{p_2}$ depending on $p_1$ and $p_2$ alone.
\end{theorem}

The proof appears in the appendix.
We can obtain simpler, but slightly more restrictive, hypotheses if we bound the fat-shattering dimension of the class of generators $\{G_\theta : \theta \in \Theta\}$ alone: 
Take the observation space $\OSpace$ to be a bounded subset of a finite-dimensional Euclidean space and the kernel to be Lipschitz continuous and translation invariant.
For the RBF kernel, the Lipschitz constant  is proportional to the inverse of the length-scale: the resulting bound loosens as the length scale shrinks. 

\section{Empirical evaluation}

In this section, we demonstrate the approach on an illustrative synthetic example as well as the standard MNIST digits and Toronto Face Dataset (TFD) benchmarks.  We show that MMD-based optimization of the generator rapidly delivers a generator that performs well in maximizing the density of a held-out test set under a kernel-density estimator.

\subsection{Gaussian data, kernel, and generator}

Under an RBF kernel and Gaussian generator with parameters  $\theta = \{\mu, \sigma \}$, it is straightforward to find the gradient of $C(\GRV_\theta,\TRV)$ by applying the chain rule.
Using fixed random standard normal numbers $\{\nrv_1,...,\nrv_M \}$,
we have 
$\grv_m= \mu + \sigma  \nrv_m$ for $m \in \{1,..,M\}$.
The result of these illustrative synthetic experiments can be found in \cref{fig:advVSmmd}. 
The dataset consisted of $N=200$ samples from a standard normal and $M=50$ noise input samples were generated from a standard normal with a fixed random seed. The algorithm was initialized at values $\{\mu,\sigma\} = \{2.5, 0.1\}$. We fixed the learning rate to $0.5$ and ran gradient descent steps for $K=250$ iterations. 

\subsection{MNIST digits}

\begin{figure*}[t]
\centering
\includegraphics[width=.8\linewidth]{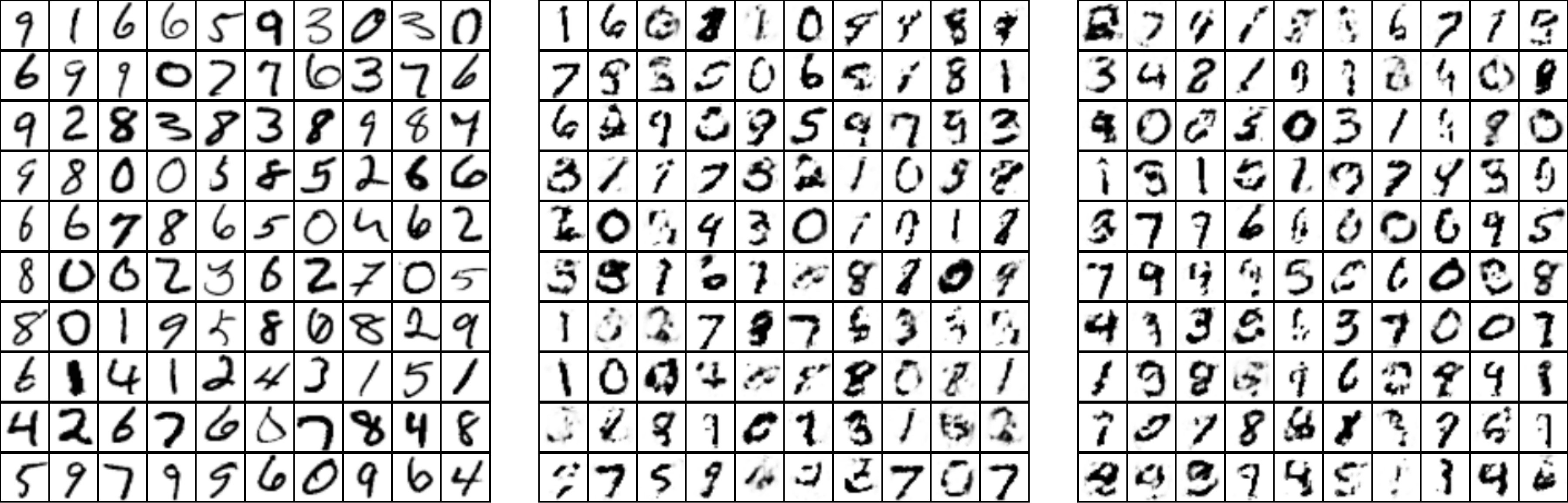}
\includegraphics[width=.8\linewidth]{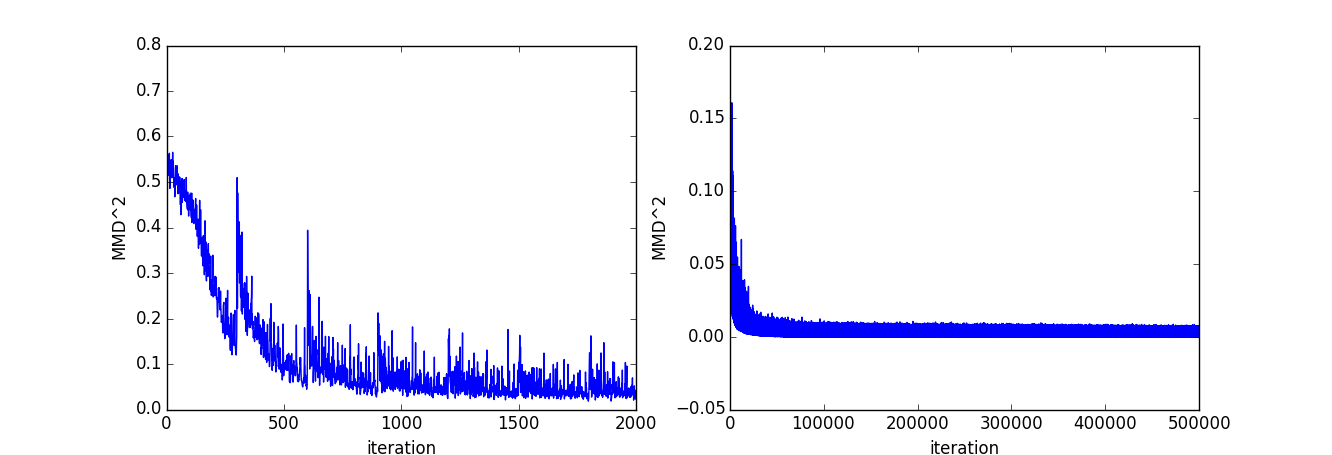}
\caption{
{\bf (top-left)} MNIST digits from the training set.
{\bf (top-right)} Newly generated digits produced after 1,000,000 iterations (approximately 5 hours). Despite the remaining artifacts, the resulting kernel-density estimate of the test data is state of the art.
{\bf (top-center)} Newly generated digits after 300 further iterations optimizing the associated empirical MMD.
{\bf (bottom-left)} MMD learning curves for first 2000 iterations.
{\bf (bottom-right)} MMD learning curves from 2000 to 500,000 iterations.  Note the difference in y-axis scale. No appreciable change is seen in later iterations.
}
\label{results}
\end{figure*}

We trained our generative network on the MNIST digits dataset \citep{Lecun98}. The generator was chosen to be a fully connected, 3 hidden layers neural network with sigmoidal activation functions. 
Following \citet{Gretton:2012}, we used a radial basis function (RBF) kernel, but also evaluated the rational quadratic (RQ) kernel \citep{ CarlGPs} and Laplacian kernel, but found that the RBF performed best in the parameter ranges we evaluated. 
We used Bayesian optimization (WHETLab) to set the bandwidth of the RBF and the number of neurons in each layer on initial test runs of 50,000 iterations. We used the median heuristic suggested by \citep{Gretton:2012} for the kernel two-sample test to choose the kernel bandwidth.
The learning rate was adjusting during optimization by RMSPROP  \citep{Tieleman2012}.

\cref{results} presents the digits learned after 1,000,000 iterations.  We performed minibatch stochastic gradient descent, resampling the generated digits every 300 iterations, using minibatches of size 500, with equal numbers of training and generated points.  It is clear that the digits produced have many artifacts not appearing in the MNIST data set.  Despite this, the mean log density of the held-out test data is  ${\bf 315 \pm 2}$, as compared with the reported ${\bf 225 \pm 2}$ mean log density achieved by adversarial nets.

There are several possible explanations for this.  First, kernel density estimation is known to perform poorly in high dimensions.  Second, the MMD objective can itself be understood as the squared difference of two kernel density estimates, and so, in a sense, the objective being optimized is directly related to the subsequent mean test log density evaluation.  There is no clear connection for adversarial networks, which might explain why it suffers under this test.  Our experience suggests that the RBF kernel delivers base line performance but that an image-specific kernel, capturing, e.g., shift invariance, might lead to better images.

\subsection{Toronto face dataset}

\begin{figure*}[t]
\centering
\includegraphics[width=.8\linewidth]{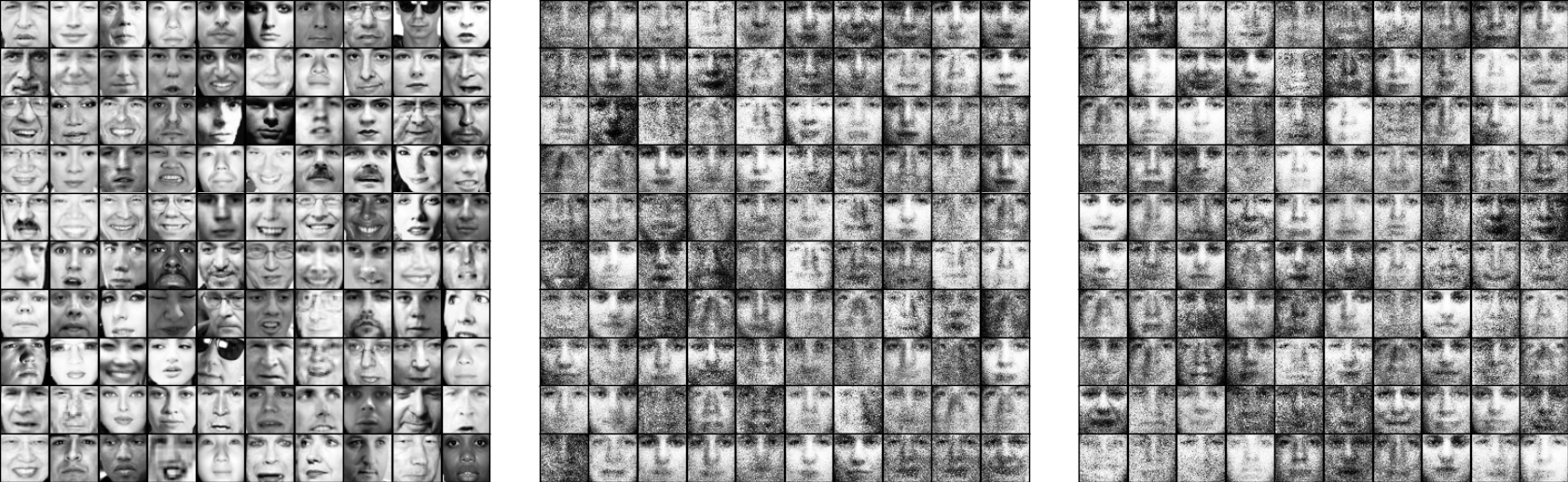}
\caption{
{\bf (left)} TFD.
{\bf (right)} Faces generated by network trained for 500,000 iterations.
{\bf (center)} Generated points after 500 iterations.
}
\label{fig:tfd faces}
\end{figure*}

We have also trained the generative MMD network on Toronto face dataset (TFD) \cite{TFD}. The parameters were adapted from the MNIST experiment: we also used a 3-hidden-layer sigmoidal MLP with similar architecture (1000, 600, and 1000 units) and RBF kernel for the cost function with the same hyper parameter. The training dataset batch sizes were equal to the number of generated points (500). The generated points were resampled every 500 iterations. The network was optimized for 500,000 iterations.

The samples from the resulting network are plotted in \cref{fig:tfd faces}.
The mean log density of the held-out test set is 2283 $\pm$ 39.  Although this figure is higher than the mean log density of 2057 $\pm$ 26 reported in \cite{Goodfellow:2014}, the samples from the MMD network are again clearly distinguishable from the training dataset. Thus the high test score suggests that kernel density estimation does not perform well at evaluating the performance for these high dimensional datasets.

\section{Conclusion}

MMD offers a closed form surrogate for the discriminator in adversarial nets framework.  After using Bayesian optimization for the parameters, we found that the network outperformed the adversarial network in terms of the density of the held-out test set under kernel density estimation. On the other hand, there is a clear discrepancy between the digits produced by MMD Nets and the MNIST digits, which might suggest that KDE is not up to the task of evaluating these models.  Given how quickly MMD Nets achieves this level of performance, it is worth considering its use as an initialization for more costly procedures.

\section*{Acknowledgments}

The authors would like to thank Bharath Sriperumbudur for technical discussions. 

\appendix

\section{Proofs}

We begin with some preliminaries and known results:
\begin{definition}[\citep{Mendelson03}]
A random variable $\radRV$ is said to be a \emph{Rademacher random variable} if it takes values in $\{-1,1\}$, each with probability $1/2$.
\end{definition}

\begin{definition}[\citep{Mendelson03}]
Let $\mu$ be a probability measure on $\mathcal{X}$, and let $\mathcal{F}$ be a class of uniformly bounded functions on $\mathcal{X}$. 
Then the \emph{Rademacher complexity} of $\mathcal F$ (with respect to $\mu$) is 
\*[
R_N(\mathcal{F}) = \EE_{\mu} \EE_{\radRV_1,\dots, \radRV_N} \left[ \frac{1}{\sqrt{N}} \sup_{f \in \mathcal{F}} \Bigl | \sum_{n=1}^{N} \radRV_n f(X_n) \Bigr | \right],
\*]
where 
$\radRV = (\radRV_1, \radRV_2,\dotsc)$ is a sequence of independent Rademacher random variables,
and $X_1,X_2,\dotsc$ are independent, $\mu$-distributed random variables, independent also from $\radRV$.
\end{definition}

\newcommand{\spX}{\mathcal{X}}
\begin{theorem}[McDiarmid’s Inequality \citep{Mendelson03}]\label{thm:McDiarmid}
Let $f :  \spX_1 \times \dotsm \times \spX_N \to \Reals$ and assume there exists $c_1,\dotsc,c_N \ge 0$ such that, for all $k \in \{1,\dots,N\}$, we have
\*[
&\sup_{x_1, \dotsc, x_k, x_k', \dotsc, x_N} |f(x_1,\dotsc,x_k,\dotsc, x_N) 
\\&\qquad\qquad\qquad\qquad- f(x_1,\dotsc,x_k',\dotsc, x_N)  | \leq c_k.
\*]
Then, for all $\varepsilon>0$ and independent random variables $\xi_1, \dotsc, \xi_n$ in $\spX$,
\*[
&\Pr \left \{  f(\xi_1,\dots,\xi_N) - \EE ( f(\xi_1,\dots,\xi_N) )\geq \varepsilon \right) \} 
         \\&\qquad< \exp \left (  \frac {-2\varepsilon^2} {\sum_{n=1}^N c_n^2} \right).
\*]
\end{theorem}

\begin{theorem}[{\citep[][Thm.~2.35]{Mendelson03}}] \label{thm:RademacherBounds}
Let $\mathcal{F} \subset B(L_\infty (\mathcal{X}))$. Assume there exists $\gamma > 1$, such that for all $\varepsilon>0$, $\fatdim (\mathcal{F}) \leq \gamma \varepsilon^{-p}$ for some $p \in \Nats$. Then there exists constants $C_p$ depending on $p$ only, such that
\[
	R_N (\mathcal{F}) \leq C_p \gamma^{\frac{1}{2}} \begin{cases} 1 &\mbox{if } 0 <p<2\\ 
\log^{\frac{3}{2}} N & \mbox{if } p =2\\ N^{\frac{1}{2} - \frac{1}{p} }&\mbox{if } p>2.\end{cases}  
\]
\end{theorem}

\begin{theorem}[\citep{Gretton:2012}] \label{thm:unbiasedMMDbound}
Assume $ 0 \leq k(x_i,x_j) \leq K$, $M=N$. Then
\[
\Pr \left[ | \MMDu^2(\RKHS,\TRV,\GRV_{\theta}) -  \MMD^2(\RKHS,\truedist,p_{\theta}) |   > \varepsilon \right] \leq \pbound
\]
where
\[
\pbound = 2\exp \left( -\frac{ \varepsilon^2 M}{ 16K^2} \right).
\]
\end{theorem}

The case where $\Theta$ is a finite set is elementary:

\begin{theorem}[estimation error for finite parameter set]\label{estimationerrorfinite}
Let $p_\theta$ be the distribution of $\genf(\NRV)$, with $\theta$ taking values in some finite set $\Theta = \{\theta_1,...,\theta_T \}, T<\infty$.
Then, with probability at least $1-(T+1) \pbound$, where $\pbound$ is defined as in \cref{thm:unbiasedMMDbound}, we have
\[
 \MMD^2(\RKHS,\truedist,p_{\empopt}) < \MMD^2(\RKHS,\truedist,p_{\classopt}) + 2 \varepsilon.
\]
\end{theorem}
\newcommand{\eMMD}{\mathcal{E}}
\newcommand{\tMMD}{\mathcal{T} }
\begin{proof}
Let  
\[
\eMMD (\theta) = \MMDu^2(\RKHS,\TRV,\GRV_{\theta}) \] 
and
\[
\tMMD (\theta) = \MMD^2(\RKHS,\truedist,p_\theta) .
\]

Note, that the upper bound stated in \cref{thm:unbiasedMMDbound} holds for the parameter value $\classopt$, i.e., 
\[ \label{inequalityforepsilonstar}
\Pr \left[ | \eMMD (\classopt ) - \tMMD (\classopt) | > \varepsilon \right] \leq  \pbound .
\]
Because $\empopt$ depends on the training data $X$ and generator data $Y$, we use a uniform bound that holds over all $\theta$.  Specifically,
\*[
\Pr \left[  | \eMMD (\empopt ) - \tMMD (\empopt) | > \varepsilon \right] 
& \leq  \Pr \left[ \sup_\theta | \eMMD (\theta ) - \tMMD (\theta) | > \varepsilon \right]  \label{InequalityEstTrueMMD} 
\\& \leq  
 \sum_{t=1}^T  \Pr \left[ | \eMMD (\empopt ) - \tMMD (\empopt) | > \varepsilon \right] \\
& \leq T \pbound .
\*]
This yields that with probability at least $1- T \pbound$,
\[
2 \varepsilon &\ge  \, | \eMMD (\empopt ) - \tMMD (\empopt) | + | \eMMD (\classopt ) - \tMMD (\classopt) |  \\ 
&\geq \, |  \eMMD (\classopt ) - \eMMD ( \empopt) + \tMMD (\empopt)  - \tMMD ( \classopt) |. \label{eq:inequalitybound}
\]
Since $\classopt$ was chosen to minimize $\tMMD(\theta)$, we know that
\[
 \tMMD (\empopt ) \geq \tMMD (\classopt ).
\]
Similarly, by \cref{eq:eMMDminimizer},
\[
  \eMMD (\classopt ) \geq \eMMD ( \empopt).
\]
Therefore it follows that 
\[
2 \varepsilon &\ge  \tMMD (\empopt ) - \tMMD (\classopt )  \\
&=   \MMD^2(\RKHS,\truedist,p_{\classopt}) - \MMD^2(\RKHS,\truedist,p_{\empopt})
\]
proving the theorem.
\end{proof}

\begin{corollary}
With probability at least $1-\delta$, 
\*[
 \MMD^2(\RKHS,\truedist,p_{\empopt}) < \MMD^2(\RKHS,\truedist,p_{\classopt}) + 2 \varepsilon_\delta,
\*]
where
\*[
\varepsilon_\delta = 8K \sqrt{\frac{1}{M}  \log \left[ 2(T+1){\delta} \right]}.
\*]
\end{corollary}

In order to prove the general result, we begin with some technical lemmas.  The development here owes much to \citet{Gretton:2012}.

\begin{lemma} \label{lem:boundFunctionExpectation}
Let
$
\mathcal{F} = \{ f: \mathcal{Y} \times \mathcal{Y} \rightarrow \Reals \}
$
and
\*[
\mathcal{F}_{+} = \{ h = f(y, \cdot) : \,f \in  \mathcal{F}, \, y \in \mathcal{Y} \} \cap B(L_\infty (\mathcal{Y})).
\*]
Let $\{Y_n\}_{n=1}^{N}$ be $\mu$-distributed independent random variables in $\mathcal{Y}$. 
Assume for some $\gamma > 1$ and some $p \in \Nats$, we have $\fatdim (\mathcal{F}_{+}) \leq \gamma \varepsilon^{-p}$, for all $\varepsilon>0$.
For $y_n \in \mathcal{Y} \quad \forall n = 1,\dots,N$, define $\rho (y_1,\dots,y_N)$ to be
\*[
\sup_{f \in \mathcal{F}} \Bigl | 
	\EE \left( f(Y,Y') \right) - \frac{1}{N(N-1)} \sum_{n\neq n'} f(y_n,y_{n'})  \Bigr |.
\*]
Then there exists a constant $C$ that depends on $p$, such that
\*[
	\EE\left(\rho(Y_1,\dots,Y_N) \right) \leq C \gamma^{\frac{1}{2}} \begin{cases}  \frac{1}{\sqrt{N-1}}  &\mbox{if } p<2\\ 
\sqrt{\frac{ \log^{3} (N-1) }{N-1}}  & \mbox{if } p =2\\ {\frac{1}{(N-1)^{\frac{1}{p}}} }&\mbox{if } p>2.\end{cases}  
\*]

\end{lemma}
\begin{proof}
Let us introduce $\{ \copynoise_n\}_{n=1}^N$, where $\copynoise_n$ and $Y_{n'}$ have the same distribution and are independent for all $n, n' \in \{1,\dots,N \}$. Then the following is true:
\*[
\EE (f(Y,Y') ) 
&= 
\EE \Bigl ( \frac{1}{N (N-1)} 
		 \sum_{n,n' :\, n\neq n'} f(\copynoise_n,\copynoise_{n'})  \Bigr )
\*]
Using Jensen's inequality and the independence of $Y,Y'$ and $Y_n, Y_{n'}$, we have
\[
&\EE\left(\rho(Y_1,\dots,Y_N)\right) \\
  \begin{split}
&	= \EE \biggl ( \,  \sup_{f \in \mathcal{F}} 
	         \biggl | \EE (f(Y,Y') ) 
\\
&\qquad\qquad\qquad     -  \frac{1}{N (N-1)} 
                  \sum_{n \neq n'}
                                                    f(Y_m, Y_{m'}) \biggr | \, \biggr)
  \end{split}   
 \label{eq:jensensfinishedhere}
 \\
 \begin{split}
&\leq \EE \biggl ( \, \sup_{f \in \mathcal{F}}
         \biggl |  \frac{1}{N (N-1)} 
		 \sum_{n\neq n'} f(\copynoise_n,\copynoise'_n) 
\\
&\qquad\qquad\qquad   -  \frac{1}{N (N-1)} 
                  \sum_{n\neq n'}
                                                    f(Y_n, Y_{n'}) \biggr | \, \biggr).
  \end{split}
\]
Introducing conditional expectations allows us to rewrite the equation with the sum over $n$ outside the expectations. I.e., \cref{eq:jensensfinishedhere} equals to
\[
 \frac{1}{N} \sum_{n} & \EE \biggl ( \EE^{(Y_n, \copynoise_n)} \Bigl ( \, 
  \sup_{f \in \mathcal{F}} 
         \Bigl |  \frac{1}{ N-1}  
 \sum_{n' \neq n'} (f(\copynoise_n,\copynoise_{n'}) - f(Y_n,Y_{n'}))  \Bigr | \,  \Bigr) \biggr) 
  \\
  \label{eq:rademacherrvssum}
  =  &\EE  \biggl ( \EE^{(Y, \copynoise)} \Bigl ( \, 
  \sup_{f \in \mathcal{F}} 
         \Bigl |  \frac{1}{ N-1}   \sum_{n=1}^{N-1} \radRV_n (f(\copynoise,\copynoise_n) - f(Y,Y_n))  \Bigr | \, \Bigr) \biggr) .
\]
The second equality follows by symmetry of random variables $\{\copynoise_n\}_{n=1}^{N-1}$. Note that we also added Rademacher random variables $\{ \radRV_n \}_{n=1}^{N-1}$ before each term in the sum since $(f(\copynoise_n,\copynoise_{n'}) - f(Y_n,Y_{n'}))$ has the same distribution as $-(f(\copynoise_n,\copynoise_{n'}) - f(Y_n,Y_{n'}))$ for all $n,n'$ and therefore the $\radRV$'s do not affect the expectation of the sum. 

Note that $\copynoise_m$ and $Y_m$ are identically distributed. Thus the triangle inequality implies that \cref{eq:rademacherrvssum}  is less than or equal to
\*[
& \frac{2}{N-1}  \EE  \left(  \EE^{(Y)} \left( \, 
  \sup_{f \in \mathcal{F}} 
         \bigl |  
         	  \sum_{n=1}^{N-1} \radRV_n  f(Y,Y_n) \bigr | \right) \right)
\\&\qquad\leq 
\frac{2}{\sqrt{N-1}} R_{N-1} (\mathcal{F}_+) ,
\*]
where $R_{N-1} (\mathcal{F}_+)$ is the Rademacher's complexity of $\mathcal{F}_+$. Then by \cref{thm:RademacherBounds}, we have 
\[
\EE\left(\rho(Y_1,\dots,Y_N)\right)  \leq C \gamma^{\frac{1}{2}} \begin{cases}  \frac{1}{\sqrt{N-1}}  &\mbox{if } p<2\\ 
\sqrt{\frac{ \log^{3} (N-1) }{N-1}}  & \mbox{if } p =2\\ {\frac{1}{(N-1)^{\frac{1}{p}}} }&\mbox{if } p>2.\end{cases}  
\]
\end{proof}

\begin{lemma} \label{lem:boundFunctionExpectationXY}
Let
$
\mathcal{F} = \{ f: \mathcal{X} \times \mathcal{Y} \rightarrow \Reals \}
$
and
\[
\mathcal{F}_{+} = \{ f: x \times \mathcal{Y} \rightarrow \Reals, \,  x \in \mathcal{X} \}.
\]
and $\mathcal{F}_{+} \subset B(L_\infty (\mathcal{Y}))$. 
Let $\{X_n\}_{n=1}^{N}$ and $\{Y_m\}_{n=1}^{M}$ be $\nu$- and $\mu$-distributed independent random variables in $\mathcal{X}$ and $\mathcal{Y}$, respectively. 
Assume for some $\gamma > 1$, such that for all $\varepsilon>0$, $\fatdim (\mathcal{F}_{+}) \leq \gamma \varepsilon^{-p}$, for some $p \in \Nats$. 
For all $x_n \in \mathcal{X}$, $ n  \le N$, and 
all $y_m \in \mathcal{Y}$, $m \le M$,
define
\*[
\rho (x_1,\dots,x_N, y_1,\dots,y_M) &=  \\
 \sup_{f \in \mathcal{F}} \Bigl | 
	\EE (f(X,Y) &- \frac{1}{N M} \sum_{n,m} f(x_n,y_m)  \Bigr | .
\*]
Then there exists $C$ that depends on $p$, such that
\*[
	&\EE\left(\rho(X_1,\dots,X_N,Y_1,\dots,Y_M)\right) 	
	\\&\qquad \leq 
		C \gamma^{\frac{1}{2}} \begin{cases}  \frac{1}{\sqrt{M}}  &\mbox{if } p<2\\ 
		\sqrt{\frac{ \log^{3} (M) }{M}}  & \mbox{if } p =2\\ {\frac{1}{(M)^{\frac{1}{p}}} }&\mbox{if } p>2.\end{cases}  
\*]
\end{lemma}

\begin{proof}
The proof is very similar to that of \cref{lem:boundFunctionExpectation}. 
\end{proof}

\begin{proof}[Proof of \cref{maintheorem}]
The proof follows the same steps as the proof of \cref{estimationerrorfinite} apart from a stronger uniform bound stated in \cref{InequalityEstTrueMMD}. I.e., we need to show:
\[
\Pr \left[ \sup_{\theta  \in \Theta} | \eMMD (\theta ) - \tMMD (\theta) | \geq \varepsilon \right] \leq \delta . 
\]
Expanding $\MMD$  as defined by \cref{eq:MMDempirical,eq:MMD}, and substituting $\GRV = \genf (\NRV)$, yields
\[ 
&\sup_{\theta \in \Theta} | \eMMD (\theta ) - \tMMD (\theta) |  
\\
\begin{split} \label{eq:supremuminequality}
&=  \sup_{\theta  \in \Theta} \Big| \EE (k(\TRV,\TRV')) \\
&\qquad\qquad - \frac{1}{N(N-1)}  
                    \sum_{n' \neq n}
                                                     k(\TRV_n,\TRV_{n'})  \\
&\qquad\qquad + \EE (k(G_{\theta} (\NRV),G_{\theta} (\NRV'))) \\
&\qquad\qquad -  \frac{1}{M (M-1)} 
                  \sum_{m\neq m'}
                                                    k(\genf (\NRV_m),\genf (\NRV_{m'})) \\
&\qquad\qquad - 2 \EE (k(\TRV,G_{\theta} (\NRV) ))  \\
&\qquad\qquad + \frac{2}{M N} \sum_{m,n} k(\TRV_n,\genf (\NRV_m)) \Big|.
\end{split}
\]
For all $n \in \{1,\dots,N\}$, $k(\TRV_n,\TRV_{n'})$ does not depend on $\theta$ and therefore the first two terms  of the equation above can be taken out of the supremum. Also, note that since $|k(\cdot,\cdot) | \leq K$, we have
\*[
\Bigl | \zeta(\trv_1,\dots, \trv_n, \dots,\trv_N)
	- \zeta(\trv_1,\dots, \trv_n', \dots,\trv_N)
\Bigr | \leq \frac{2 K}{N} ,
\*]
where
\*[
\zeta(\trv_1,\dots,\trv_N) = \frac{1}{N(N-1)}  
           \sum_{n,n' :\, n' \neq n} 
                          k(\trv_n,\trv_{n'}),
\*]
and $\zeta$ is an unbiased estimate of $\EE (k(\TRV,\TRV'))$.
Then from McDiarmid's inequality on $\zeta$, we have
\[ 
\begin{split}
&\Pr \Bigr(  \Bigl | \EE (k(\TRV,\TRV')) 
	- \frac{1}{N(N-1)}  
		\sum_{n' \neq n}
			k(\TRV_n,\TRV_{n'}) \Bigr | \geq \varepsilon \Bigr)   
			\\ & \qquad \leq
	\exp \left(-\frac{\varepsilon^2}{2 K^2} N \right).
	\end{split} \label{eq:differencekernels}
\]

Therefore \cref{eq:supremuminequality} is bounded by the sum of the bound on \cref{eq:differencekernels} and the following:
\[ \label{eq:uniformbound}
\begin{split}
& \sup_{\theta \in \Theta} \Big| \EE (k(G_{\theta} (\NRV),G_{\theta} (\NRV'))) 
 \\&\qquad     -  \frac{1}{M (M-1)} 
                  \sum_{m \neq m'}
                                                    k(\genf (\NRV_m),\genf (\NRV_{m'})) \\
&\qquad - 2 \EE (k(\TRV,G_{\theta} (\NRV) )) 
 \\&\qquad
                + \frac{2}{M N} \sum_{m,n} k(\TRV_n,\genf (\NRV_m)) \Big| .
\end{split}
\]
Thus the next step is to find the bound for the supremum above.

Define
\*[
& f(\NRV_1,\dots,\NRV_M ; \noisedist)  =  f(\nvec_{M}) \\
&= \sup_{\theta \in \Theta} \Big| \EE (k(G_{\theta} (\NRV),G_{\theta} (\NRV'))) 
      \\
 & \qquad -  \frac{1}{M (M-1)} 
                  \sum_{m \neq m' }
                                                    k(\genf (\NRV_m),\genf (\NRV_{m'})) \Big|
\*]
and
\*[
& h(\TRV_1,\dots,\TRV_N, \NRV_1,\dots,\NRV_M; \truedist, \noisedist) 
\\&=  h(\tvec_{N}, \nvec_{M}) \\
&= \sup_{\theta \in \Theta}   \Bigl | \frac{1}{M N} \sum_{m,n} k(\TRV_n,\genf (\NRV_m))
-  \EE (k(\TRV,G_{\theta} (\NRV) )) \Bigr |.
\*]
Then by triangle inequality
\[
\cref{eq:uniformbound} & \leq f(\nvec_{M}) + 2h(\tvec_{N}, \nvec_{M}).
\]
We will first find the upper bound on $ f(\nvec_{M})$, i.e., for every $\varepsilon>0$, we will show that there exists $\delta_f$, such that
\[ \label{ineq:fbound}
\Pr \left( f(\nvec_M ) > \varepsilon \right) \leq \delta_f
\]
For each $m \in \{1,\dots,M\}$,
\[
&\Big| f(\NRV_1,\dots, \NRV_m, \dots,\NRV_M) \\ 
& \qquad- f(\NRV_1,\dots, \NRV_m ', \dots,\NRV_M) \Big| \leq \frac{2K}{M}
\]
since the kernel is bounded by $K$, and therefore  $k(\genf (\NRV_m),\genf (\NRV_{m'}))$ is bounded by $K$ for all $m$. The conditions of \cref{thm:McDiarmid} are satisfied and thus we can use McDiarmid’s Inequality on $f$:
\[
\Pr \left( f(\nvec_M) - \EE (f(\nvec_M)) \geq \epsilon \right)  \leq \exp \left( - \frac{\varepsilon^2 M}{2 K^2} \right) .
\]
Define
\[
\GKclass = \{ k(\genf (\cdot),\genf (\cdot)) : \, \theta \in \Theta  \}
\]
To show \cref{ineq:fbound}, we need to bound the expectation of $f$. We can apply \cref{lem:boundFunctionExpectation} on the function classes $\GKclass$ and $\GKZclass$. The resulting bound is
\[ \label{inequality:p1 rate}
\EE (f(\nvec_M)) \leq \varepsilon_{p1} = 
C_f \gamma_{1}^{\frac{1}{2}} 
	\begin{cases}  \frac{1}{\sqrt{M-1}}  &\mbox{if } p_1<2\\ 
		\sqrt{\frac{ \log^{3} (M-1) }{M-1}}  & \mbox{if } p_1 =2\\ 
		{\frac{1}{(M-1)^{\frac{1}{p_1}}} }&\mbox{if } p_1>2.\end{cases}  ,
\]
where  $p_1$ and $\gamma_1$ are parameters associated to fat shattering dimension of $\GKZclass$ as stated in the assumptions of the theorem, and $C_f$ is a constant depending on $p_1$. 
 
Now we can write down the bound on $f$:
\[ \label{eq:boundf}
\Pr \left( f(\nvec_M) 
	\geq  \varepsilon_{p_1}  + \epsilon \right)  
		\leq \exp \left( - \frac{\varepsilon^2 M}{2 K^2} \right)  = \delta_f.
\]

Similarly, $h(\tvec_{N}, \nvec_{M})$ has bounded differences:
\[
\begin{split}
\Big|  & h(\TRV_1,\dots,\TRV_n,\dots,\TRV_N,  \NRV_1, \dots,\NRV_M) \\
	& - h(\TRV_1,\dots,\TRV_{n'},\dots,\TRV_N, \NRV_1,\dots,\NRV_M) \Big| \leq \frac{2K}{N}
\end{split}
\]
and
\[
\begin{split}
\Big|  & h(\TRV_1,\dots,\TRV_N,  \NRV_1,\dots, \NRV_m, \dots,\NRV_M) \\
	& - h(\TRV_1,\dots,\TRV_N, \NRV_1,\dots, \NRV_{m'}, \dots,\NRV_M) \Big| \leq \frac{2K}{M}.
\end{split}
\]
McDiarmid's inequality then implies 
\[ \label{eq:boundh}
\begin{split}
&\Pr \left(h(\tvec_{N}, \nvec_{M}) - \EE (h(\tvec_{N}, \nvec_{M}) \geq 
	\varepsilon \right)  \\
	& \quad\qquad\qquad\qquad\qquad \leq \exp \left( - \frac{\varepsilon^2 }{2 K^2} \frac{NM}{N+M}\right ) .
\end{split}
\]
We can bound expectation of $h(\tvec_{N}, \nvec_{M})$ using \cref{lem:boundFunctionExpectationXY} applied on $\GKclassX$ and $\GKZclassX$, where
\[
\GKclassX = \{ k(\cdot,\genf (\cdot)) : \, \theta \in \Theta  \}.
\]
Then
\[ \label{inequality:p2 rate}
\EE (h(\tvec_{N}, \nvec_{M})) \leq \varepsilon_{p_2}  = 
		C_h \gamma_{2}^{\frac{1}{2}} 
			\begin{cases}  
			        \frac{1}{\sqrt{M}}  &\mbox{if } p_2<2\\ 
				\sqrt{\frac{ \log^{3} (M) }{M}}  & \mbox{if } p_2 =2\\ 
				{\frac{1}{M^{\frac{1}{p_2}}} }&\mbox{if } p_2>2.
                          \end{cases} 
\] 
for some constant $C_h$ that depends on $p_@$. The final bound on $h$ is then
\[
\begin{split}
&\Pr \left(h(\tvec_{N}, \nvec_{M}) \geq 
		 \varepsilon_{p_2}  +	\varepsilon \right) \\
		 &\qquad\qquad\qquad\quad \leq \exp \left( - \frac{\varepsilon^2 }{2 K^2} \frac{NM}{N+M}\right )  = \delta_h .
\end{split}
\]

Summing up the bounds from \cref{eq:boundf} and \cref{eq:boundh}, it follows that
\[
\begin{split}
&\Pr \left( f(\nvec_{M}) + 2h(\tvec_{N}, \nvec_{M}) \geq \varepsilon_{p1} + 2 \varepsilon_{p_2} + 3 \varepsilon \right) \\
 &\qquad\qquad\qquad\qquad\qquad\qquad\quad \leq \max (\delta_f, \delta_h) = \delta_h.
\end{split}
\]

Using the bound in \cref{eq:differencekernels}, we have obtain the uniform bound we were looking for:
\[
\Pr \left[ \sup_{\theta  \in \Theta} | \eMMD (\theta ) - \tMMD (\theta) | > \varepsilon_{p_1} + 2 \varepsilon_{p_2} + 4 \varepsilon  \right] \leq \delta_h,
\] 
which by \cref{InequalityEstTrueMMD} yields
\[
\Pr \left[ | \eMMD (\hat{\theta} ) - \tMMD (\hat{\theta}) | > \varepsilon_{p_1} + 2 \varepsilon_{p_2} + 4 \varepsilon  \right] \leq \delta_h.
\] 
Since it was assumed that $K=1$ and $N=M$, we get
\[
\delta_h = \exp \left( - \frac{\varepsilon^2 M}{4} \right).
\] 

To finish, we proceed as in the proof of  \cref{estimationerrorfinite}. We can rearrange some of the terms to get a different form of \cref{inequalityforepsilonstar}:
\*[ 
&\Pr \left[ | \eMMD (\classopt ) - \tMMD (\classopt) | >  2 \varepsilon \right] 
\\	&\qquad\leq  2 \exp \left( -  \frac{\varepsilon^2 M }{4}  \right)  = 2 \delta_h.
\*] 
All of the above implies that for any $\varepsilon>0 $, there exists $\delta$, such that
\*[
&\Pr \bigl (  \MMD^2(\RKHS,\truedist,p_{\empopt}) 
\\
&\qquad\qquad\qquad - \MMD^2(\RKHS,\truedist,p_{\classopt}) 
	\geq \varepsilon \bigr )
		\leq \delta ,
\*]
where
\[
\varepsilon = \varepsilon_{p_1} + 2 \varepsilon_{p_2}   
		+  \frac{12}{\sqrt{M}} \sqrt { \log \frac{2}{\delta}}.
\]
We can rewrite $\varepsilon$ as: 
\[
\varepsilon 
 &= r(p_1,\gamma_1,M) + r(p_2,\gamma_2,M-1) + 12 M^{-\frac 1 2} \sqrt { \log \frac{2}{\delta}},
\]
The rate $r(p,\gamma,N)$  is given by \cref{inequality:p1 rate} and  \cref{inequality:p2 rate}:
\[ 
r(p,\gamma,M) = C_{p} \sqrt{\gamma}
\begin{cases}  
        M^{- \frac 1 2}        &\text{if } p<2,\\ 
        M^{- \frac 1 2} \log^{\frac 3 2} (M)   &\text{if } p = 2, \\ 
	M^{-\frac{1}{p}}   &\text{if } p>2,
\end{cases}
\] 
where the constants $C_{p1}$ and $C_{p2}$ depend on $p_1$ and $p_2$ alone.
\end{proof}

We close by noting that the approximation error is zero in the nonparametric limit.
\begin{theorem}[{\citet{Gretton:2012}}] \label{thm:uniMMD}
Let $F$ be the unit ball in a universal RKHS $\RKHS$, defined on the compact metric space $\OSpace$, with associated continuous kernel $k(\cdot, \cdot)$. Then $\MMD[\RKHS, p, q] = 0$ if and only if $p = q$.
\end{theorem}

\begin{corollary}[approximation error]
Assume $\truedist$ is in the family $\{p_\theta\}$ and that $\RKHS$ is an RKHS induced by a characteristic kernel.  Then 
\[
\inf_{\theta} \MMD(\RKHS,\truedist,p_{\theta}) = 0
\]
and the infimum is achieved at $\theta$ satisfying $p_\theta = \truedist$.
\end{corollary}
\begin{proof}
By \cref{thm:uniMMD}, it follows that $\MMD^2 (\RKHS,\cdot,\cdot)$ is a metric.  The result is then immediate.
\end{proof}

\newpage
\printbibliography[minnames=99,maxnames=100]

\end{document}